\documentclass[svgnames]{amsart}
\usepackage[a-1b]{pdfx}
\usepackage{
	bm,
	amssymb,
	mathtools,
	xcolor,
	booktabs,
	hyperref,
	graphicx,
	comment,
}
\usepackage[hhmmss]{datetime}
\usepackage[noabbrev, capitalise]{cleveref}
\usepackage{tikz}
\usetikzlibrary{matrix, fit}
\usetikzlibrary{backgrounds}
\usetikzlibrary{arrows,positioning}
\newcommand{\x}{\times}
\renewcommand{\o}{\phantom{\times}}

\hypersetup{colorlinks, allcolors=MediumBlue}

\crefdefaultlabelformat{#2\textup{#1}#3} 

\newtheorem{thm}  {Theorem} [section]
\newtheorem{prop} [thm] {Proposition}
\newtheorem{lem} [thm] {Lemma}
\theoremstyle{definition} 
\newtheorem{nts} [thm] {Notation}
\newtheorem{dfn} [thm] {Definition}

\numberwithin{equation}{section}

\crefname{thm}{Theorem}{Theorems}
\crefname{lem}{Lemma}{Lemmas}
\crefname{prop}{Proposition}{Propositions}
\crefname{nts}{Notation}{Notations}
\crefname{dfn}{Definition}{Definitions}
\crefname{fig}{Figure}{Figures}
\crefname{tab}{Table}{Tables}

\DeclarePairedDelimiterX\set[1]\lbrace\rbrace{\,\def\given{\mid}#1\,} 
\DeclarePairedDelimiterX\ip[1][]{{#1}}
\DeclarePairedDelimiterX\fp[1]\{\}{{#1}}
\DeclarePairedDelimiterX\size[1]\lvert\rvert{{#1}}

\newcommand{\ind}{\ensuremath{\bm{1}}}
\newcommand{\MM}{\ensuremath{\bm{M}}}
\newcommand{\GG}{\ensuremath{\bm{G}}}
\newcommand{\II}{\ensuremath{\bm{I}}}
\newcommand{\KK}{\ensuremath{\bm{K}}}
\renewcommand{\P}{P}
\newcommand{\BB}[1]{\mathfrak{B}(#1)}
\newcommand{\E}[1]{E(#1)}
\newcommand{\dist}{\sim}
\newcommand{\dK}{\kappa}
\newcommand{\formal}{\textup{(}formal\textup{)}}

\begin{document}

\title[On concepts of random contexts]{On formal concepts of random formal contexts}
\author[T.\ Sakurai]{\href{https://orcid.org/0000-0003-0608-1852}{Taro Sakurai}}
\address{Department of Mathematics and Informatics, Graduate School of Science, Chiba University, 1-33, Yayoi-cho, Inage-ku, Chiba-shi, Chiba, 263-8522 Japan}
\email{tsakurai@math.s.chiba-u.ac.jp}

\subjclass[2010]{%
	\href{https://zbmath.org/classification/?q=68T30}{68T30} (%
	\href{https://zbmath.org/classification/?q=06B99}{06B99},
	\href{https://zbmath.org/classification/?q=05C80}{05C80},
	\href{https://zbmath.org/classification/?q=60C05}{60C05})}
\date{\today}
\begin{abstract}
	\noindent 
	In formal concept analysis,
	it is well-known that the number of formal concepts can be exponential in the worst case.
	To analyze the average case, we introduce a probabilistic model for random formal contexts and
	prove that the average number of formal concepts has a superpolynomial asymptotic lower bound.
\end{abstract}

\maketitle

\setcounter{tocdepth}{1}
\tableofcontents

\section{Introduction}
\noindent 
How many formal concepts does a formal context have?
This is one of the fundamental problems in the theory of \emph{formal concept analysis}---%
an application area of lattice theory which originates from Wille~\cite{Wil82} to support \emph{data analysis} and \emph{knowledge processing}.
In the graph-theoretic language, the problem asks the number of maximal bicliques of bipartite graphs.
The problem of determining the number of formal concepts is proved to be \#P-complete by Kuznetsov \cite[Theorem 1]{Kuz01}.
Even though the counting problem is hard in general,
it is of interest to get a general idea of how large the number is.

It is well-known that
the number of formal concepts can be exponential in the worst case,
and it can be one in the best case.
Such extremal formal contexts are obtained from contranomial scales and formal contexts defined by the empty relation.
Since these examples appear to be highly atypical,
it is natural to study the number of formal concepts in the \emph{average} case.
To this end, we introduce random formal contexts (\cref{dfn:rc})
and present an exact formula for the average number of formal concepts (\cref{prop:avg}).
Lastly,
we prove that the average number of formal concepts has a \emph{superpolynomial} asymptotic lower bound (\cref{thm:main}),
which is the main result of this article.
Our theorem and its proof help to understand why a ``typical'' formal context has numerous formal concepts.

\section{Preliminaries}
\subsection{Formal concept analysis}
We recall basic notions in formal concept analysis which can be found in the textbook by Ganter and Wille~\cite[Chapter~1]{GW99}.
A \emph{{\formal} context} is defined to be a triple \( K = (G, M, I) \) consists of
two sets \( G \), \( M \), and a subset \( I \) of \( G \times M \).
An element \( g \) of \( G \) is called an \emph{object},
an element \( m \) of \( M \) is called an \emph{attribute},
and \( I \) is called the \emph{incidence relation} of the context \( K \).
An object \( g \) is said to \emph{have} an attribute \( m \) if a pair \( (g, m) \) belongs to \( I \).
A context is often represented by a \emph{cross table}
whose rows and columns are indexed by objects and attributes,
and the incidence relation is indicated by crosses as in \cref{fig:cross}.
\begin{figure}[h]
	\begin{tikzpicture}
		\matrix[
			matrix of math nodes,
			row sep=.5ex,
			column sep=.5ex,
			left delimiter=., right delimiter=. ,
			nodes={text width=.75em, text height=1.75ex, text depth=.5ex, align=center}
			] (m)
			{
			\o & \o & \o & m  & \o & \o \\
			\o & \o & \o & \o & \o & \o \\
			\o & \o & \o & \o & \o & \o \\
			g  & \o & \o & \x & \o & \o \\
			\o & \o & \o & \o & \o & \o \\
			\o & \o & \o & \o & \o & \o \\
			};
			\draw (m-6-2.south west) rectangle (m-2-6.north east);
			\begin{scope}[on background layer]
				\node[fit=(m-4-2)(m-4-6), draw=MediumBlue!5, fill=MediumBlue!5, rounded corners] {};
				\node[fit=(m-2-4)(m-6-4), draw=MediumBlue!5, fill=MediumBlue!5, rounded corners] {};
				\node[fit=(m-4-4), fill=MediumBlue!10] {};
			\end{scope}
	\end{tikzpicture}
	\caption{The cross table of a context.}
	\label{fig:cross}
\end{figure}
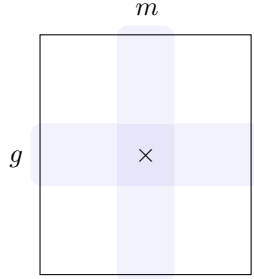

Let \( A \) be a set of objects and let \( B \) be a set of attributes.
The set of attributes that all objects in \( A \) have in common is denoted by
\begin{align*}
	A' &= \bigcap_{g \in A} \set{ m \in M \given (g, m) \in I }.  \\
	\intertext{Similarly, the set of objects that have all attributes in \( B \) is denoted by}
	B' &= \bigcap_{m \in B} \set{ g \in G \given (g, m) \in I }.
\end{align*}
A pair \( (A, B) \) is defined to be a \emph{{\formal} concept} if \( A' = B \) and \( B' = A \);
the first and second components are called the \emph{extent} and \emph{intent} of the concept.
The set of concepts of a context \( K \) is denoted by \( \BB{K} \).

\subsection{Asymptotic analysis}
We recall two useful notations in asymptotic analysis: the \emph{little-oh notation} and the \emph{Vinogradov notation}.
Let \( (x_n) \) and \( (y_n) \) be real sequences.
For an arbitrary positive real number \( \varepsilon \), if
\( \size{x_n} < \varepsilon \size{y_n} \)
for sufficiently large \( n \), then we write \( x_n = o(y_n) \).
If there is some positive real number \( \gamma \) satisfying
\( \size{x_n} \le \gamma \size{y_n} \)
for sufficiently large \( n \), then we write \( y_n \gg x_n \).

\section{Random contexts}
In this section, we introduce a probabilistic model for random contexts.
Although we provide its measure-theoretic formalization later for completeness,
the randomness we consider might be best described by the following informal manner.

Let \( n \) be a positive integer and take an \( n \)-set, say \( U = \{1, 2, \dotsc, n \} \).
For each element of \( U \),
we regard it as an \emph{object} with probability \( p \) and as an \emph{attribute} with probability \( 1 - p \), independently.
Subsequently, for each pair \( (g, m) \) of an object \( g \) and an attribute \( m \),
we regard an object \( g \) \emph{has} an attribute \( m \) with probability \( q \),
independently.
We add that the probabilities \( p \) and \( q \) are not necessarily constants like \( p = 1/2 \) and may be functions of \( n \) like \( q = 1 - 1/n \).

A similar probabilistic model with a fixed number of objects and attributes is used by Kov\'acs in \cite[\S2.1]{Kov18} to estimate the number of concepts.
Those who familiar with random graph theory would instantly recognize that this is very much alike to the model for binomial random graphs~\cite[p.~2]{JLR00},
which is also known as the Erd\H{o}s-R\'enyi model.
In this article, we content ourselves with this simplest model for random contexts.
The readers may wish to skim through the next notation and definition
if they are comfortable with this informal description of our probabilistic model.

Throughout this article, we use a convention to write random variables in bold.
For basic concepts of probability theory, we refer the readers to a work by Bauer~\cite[Chapter~I]{Bau96}, for example.

\begin{nts}
	Let \( n \) be a positive integer and let \( p \) and \( q \) be real numbers belonging to the unit interval \( [0, 1] \).
	Set \( U = \{1, 2, \dotsc, n\} \).
	Write \( \Omega \) for the set of contexts \( (G, M, I) \) with \( G + M = U \) where \( + \) denotes the disjoint union.
	Define the probability measure \( \P = \dK_{n, p, q} \) on the power set \( 2^\Omega \) by
	\begin{equation*}
		\P\{(G, M, I)\} = p^{\size{G}}(1 - p)^{\size{M}} \, q^{\size{I}}(1 - q)^{\size{G \times M - I}}.
	\end{equation*}
\end{nts}

The probability space \( (\Omega, 2^\Omega, \P) \) is our mathematical model for random contexts.

\begin{dfn}
	\label{dfn:rc}
	We call an \( \Omega \)-valued random variable \( \KK \) a \emph{random context} and write \( \KK \dist \dK_{n, p, q} \)
	if the distribution of \( \KK \) equals \( \dK_{n, p, q} \).
\end{dfn}

For a real-valued function \( f \) on \( \Omega \) and a random context \( \KK \),
we write
\begin{equation*}
	\E{f \circ \KK} = \int f \, d\P = \sum_{K \in \Omega} f(K) \P\{ K \}
\end{equation*}
for the expectation.

\section{Average number of concepts}
Based on the notion of random contexts that is introduced in the previous section,
we show an exact formula for the average number of concepts in this section.
\begin{prop}
	\label{prop:avg}
	Let \( \KK \) be a random context with \( \KK \dist \dK_{n, p, q} \).
	Then
	\begin{equation}
		\E{\size{\BB{\KK}}} = \sum_{(a, b, c, d)} \binom{n}{a \; b \; c \; d} \, p^{a + c}(1 - p)^{b + d} \, q^{ab}(1 - q^a)^d(1 - q^b)^c
		\label{eq:ebk}
	\end{equation}
	where the sum is taken over all non-negative integers with \( a + b + c + d = n \).
\end{prop}
\begin{proof}
	Set \( \KK = (\GG, \MM, \II) \).
	Let \( A \) and \( B \) be subsets of \( U \).
	We write \( \ind_{\{ (A, B) \in \BB{\KK} \}} \) for the indicator variable of an event that a pair \( (A, B) \) is a concept of \( \KK \).
	By the linearity of expectation and the law of total probability, we may reduce the problem as
	\begin{align*}
		\E{\size{\BB{\KK}}}
		&= \sum_{(A, B)} \E{\ind_{\{ (A, B) \in \BB{\KK} \}}}
		= \sum_{(A, B)} \P\{ (A, B) \in \BB{\KK} \} \\
		&= \sum_{(A, B, C, D)} \P(\{ (A, B) \in \BB{\KK} \} \cap \{ \GG = A + C \} \cap \{ \MM = B + D \})
	\end{align*}
	where the sums are taken over all tuples of subsets of \( U \).
	Suppose that \( (A, B, C, D) \) is an ordered partition of the set \( U \).
	From the reduction, it is enough to show that
	\begin{multline*}
		\quad \P(\{ (A, B) \in \BB{\KK} \} \cap \{ \GG = A + C \} \cap \{ \MM = B + D \}) \hfill \\
		\hfill{} = p^{\size{A + C}}(1 - p)^{\size{B + D}} \, q^{\size{A \times B}}(1 - q^{\size{A}})^{\size{D}}(1 - q^{\size{B}})^{\size{C}}. \quad
	\end{multline*} 
	The cross table of a context in \cref{fig:prob} may help the readers to see why this claim holds.
	\begin{figure}[h]
		\begin{tikzpicture}
			\matrix[
				matrix of math nodes,
				row sep=.5ex,
				column sep=.5ex,
				left delimiter=., right delimiter=. ,
				nodes={text width=.75em, text height=1.75ex, text depth=.5ex, align=center}
				] (m)
				{
				\o & \o & B      & \o & \o & D      & \o \\
				\o & \o & \o     & \o & \o & \o     & \o \\
				A  & \o & \x     & \o & \o & \vdots & \o \\
				\o & \o & \o     & \o & \o & \o     & \o \\
				\o & \o & \o     & \o & \o & \o     & \o \\
				C  & \o & \cdots & \o & \o & \ast   & \o \\
				\o & \o & \o     & \o & \o & \o     & \o \\
				};
				\draw[dotted] (m-4-2.south west) -- (m-4-7.south east);
				\draw[dotted] (m-2-4.north east) -- (m-7-4.south east);
				\draw (m-7-2.south west) rectangle (m-2-7.north east);
				\begin{scope}[on background layer]
					\node[fit=(m-3-3), fill=MediumBlue!5, rounded corners] {};
					\node[fit=(m-6-2)(m-6-4), fill=MediumBlue!5, rounded corners] {};
					\node[fit=(m-2-6)(m-4-6), fill=MediumBlue!5, rounded corners] {};
					\node[fit=(m-6-6), fill=MediumBlue!5, rounded corners] {};
				\end{scope}
		\end{tikzpicture}
		\caption{When \( \{ (A, B) \in \BB{\KK} \} \cap \{ \GG = A + C \} \cap \{ \MM = B + D \} \) occurs.}
		\label{fig:prob}
	\end{figure}
	First, every element of \( A + C \) must belong to \( \GG \) with probability \( p^{\size{A + C}} \) (row header), and
	every element of \( B + D \) must belong to \( \MM \) with probability \( (1 - p)^{\size{B + D}} \) (column header).
	Second, every pair of \( A \times B \) must belong to \( \II \) with probability \( q^{\size{A \times B}} \) (upper-left corner).
	Next, every attribute in \( D \) must not be shared by all objects in \( A \) with probability \( (1 - q^{\size{A}})^{\size{D}} \) (upper-right corner), and
	every object in \( C \) must not have all attributes in \( B \) with probability \( (1 - q^{\size{B}})^{\size{C}} \) (lower-left corner).
	Last, the rest entries (lower-right corner) do not affect the occurrence of the event.
	The above argument establishes the claim and completes the proof.
\end{proof}

\section{Asymptotic lower bound}
In this section,
we study random contexts with constant probabilities \( p = q = 1/2 \) in detail and
prove that the average number of concepts has a superpolynomial asymptotic lower bound.
The following is the main result of this article.

\begin{thm}
	\label{thm:main}
	Let \( (\KK_n) \) be a sequence of random contexts with \( \KK_n \dist \dK_{n, \frac{1}{2}, \frac{1}{2}} \).
	Then
	\begin{equation*}
		\E{\size{\BB{\KK_n}}} > n^{\log n}
	\end{equation*}
	for sufficiently large \( n \).
	In particular, \( \E{\size{\BB{\KK_n}}} \gg n^{\log n} \).
\end{thm}

For a real number \( x \), the integer part and fractional part of \( x \) are denoted by \( \ip{x} \) and \( \fp{x} \).
To obtain a lower bound for the average number of concepts of \( \KK_n \),
we \emph{single out} the specific term
\begin{align}
	\label{eq:term}
	t_n &= \binom{n}{a_n \; b_n \; c_n \; d_n} \, p^{a_n + c_n}(1 - p)^{b_n + d_n} \, q^{a_nb_n}(1 - q^{a_n})^{d_n}(1 - q^{b_n})^{c_n}
\end{align}
in \eqref{eq:ebk} for constant probabilities \( p = q = 1/2 \) where
\begin{align}\begin{split}
	\label{eq:abcd}
	a_n &= \ip[\bigg]{\frac{\log n}{\log 2}}, \qquad
	b_n = \ip[\bigg]{\frac{\log n}{\log 2}} + 2\fp[\bigg]{\frac{n}{2}}, \qquad \text{and} \\
	c_n &= d_n = \ip[\bigg]{\frac{n}{2}} - \ip[\bigg]{\frac{\log n}{\log 2}}.
\end{split}\end{align}
Although this is just one term in the summation, it turns out to be large enough for our purpose.
The asymptotic behavior of \( t_n \) is described as follows.
\begin{lem}
	\label{lem:equiv}
	With notation in \eqref{eq:term},
	\begin{equation*}
		\log t_n = \frac{\log^2 n}{\log 2}\left(1 + o(1)\right).
	\end{equation*}
\end{lem}
To prove this asymptotic equivalence, we need some lemmas.
\begin{lem}
	\label{lem:mult}
	With notation in \eqref{eq:abcd},
	\begin{equation*}
		\log \binom{n}{a_n \; b_n \; c_n \; d_n} = n \log 2 + 2 \frac{\log^2 n}{\log 2} + o(\log^2 n).
	\end{equation*}
\end{lem}

\begin{proof}
	By the Stirling formula and the Taylor formula,
	\begin{align*}
		\log n!
		&= n \log n - n + o(\log^2 n), \\
		\log a_n!
		&= \log\left(\frac{\log n}{\log 2}  - \fp[\bigg]{\frac{\log n}{\log 2}} \! \right)! = o(\log^2 n), \\
		\log b_n!
		&= \log\left(\frac{\log n}{\log 2}  - \fp[\bigg]{\frac{\log n}{\log 2}} + 2\fp[\bigg]{\frac{n}{2}} \! \right)! = o(\log^2 n), \qquad \text{and} \\
		\log c_n!
		&= \log d_n!
		= \log\left(\frac{n}{2} - \fp[\bigg]{\frac{n}{2}} - \frac{\log n}{\log 2}  + \fp[\bigg]{\frac{\log n}{\log 2}} \! \right)! \\
		&= \log\left(\frac{n}{2} - \frac{\log n}{\log 2}  + o(\log n)\right)! \\
		&= \left(\frac{n}{2} - \frac{\log n}{\log 2}  + o(\log n)\right) \log\left(\frac{n}{2} - \frac{\log n}{\log 2}  + o(\log n)\right) \\
		&\qquad - \left(\frac{n}{2} - \frac{\log n}{\log 2}  + o(\log n)\right) + o(\log^2 n) \\
		&= \left(\frac{n}{2} - \frac{\log n}{\log 2}\right) \log\left(\frac{n}{2} - \frac{\log n}{\log 2}  + o(\log n)\right) - \frac{n}{2} + o(\log^2 n) \\
		&= \left(\frac{n}{2} - \frac{\log n}{\log 2}\right) \! \left(\log n -\log 2 - \frac{2}{n}\frac{\log n}{\log 2} + o\left(\frac{\log^2 n}{n}\right) \! \right)
		        - \frac{n}{2} + o(\log^2 n) \\
		&= \frac{1}{2}n\log n - \frac{1}{2}(1 + \log 2)n - \frac{\log^2 n}{\log 2} + o(\log^2 n).
	\end{align*}
	Therefore
	\begin{align*}
		&\log \binom{n}{a_n \; b_n \; c_n \; d_n} \\
		&\qquad\quad = n \log n - n - 2\left(\frac{1}{2}n\log n - \frac{1}{2}(1 + \log 2)n - \frac{\log^2 n}{\log 2}\right) + o(\log^2 n) \\
		&\qquad\quad = n \log 2 + 2 \frac{\log^2 n}{\log 2} + o(\log^2 n).
		\qedhere
	\end{align*}
\end{proof}

\begin{lem}
	\label{lem:bdd}
	With notation in \eqref{eq:abcd},
	\begin{equation*}
		\size[\big]{\log (1 - 2^{-a_n})^{d_n}(1 - 2^{-b_n})^{c_n}} < 2.
	\end{equation*}
\end{lem}
\begin{proof}
	We may assume that \( n > 2 \).
	Note that \( c_n = d_n \le n/2 \) and
	\begin{align*}
		1 - 2^{-b_n}
		\ge 1 - 2^{-a_n}
		= 1 - 2^{-\frac{\log n}{\log 2} + \fp{\frac{\log n}{\log 2}}}
		= 1 - \frac{2^{\fp{\frac{\log n}{\log 2}}}}{n}
		> 1 - \frac{2}{n}.
	\end{align*}
	Hence
	\begin{align*}
		&\size[\big]{\log (1 - 2^{-a_n})^{d_n}(1 - 2^{-b_n})^{c_n}} \\
		&\qquad\qquad= -d_n \log (1 - 2^{-a_n}) - c_n \log (1 - 2^{-b_n})
		< -n \log \left(1 - \frac{2}{n}\right)
		\le 2.
		\qedhere
	\end{align*}
\end{proof}

\begin{proof}[Proof of \cref{lem:equiv}]
	By \cref{lem:mult,lem:bdd},
	\begin{align*}
		\log t_n
		&= \log \binom{n}{a_n \; b_n \; c_n \; d_n}
		         - \left(\frac{n}{2} - \fp[\bigg]{\frac{n}{2}} \! \right) \log 2
		         - \left(\frac{n}{2} + \fp[\bigg]{\frac{n}{2}} \! \right) \log 2 \\
		& \qquad - \left(\frac{\log n}{\log 2} - \fp[\bigg]{\frac{\log n}{\log 2}} \! \right) \! \left(\frac{\log n}{\log 2} - \fp[\bigg]{\frac{\log n}{\log 2}} + 2\fp[\bigg]{\frac{n}{2}} \! \right) \log 2 \\
		& \qquad + \log (1 - 2^{-a_n})^{d_n}(1 - 2^{-b_n})^{c_n} \\
		&= n \log 2 + 2 \frac{\log^2 n}{\log 2} - \frac{n}{2} \log 2 - \frac{n}{2} \log 2 - \frac{\log^2 n}{\log 2} + o(\log^2 n) \\
		&= \frac{\log^2 n}{\log 2} + o(\log^2 n) = \frac{\log^2 n}{\log 2}\left(1 + o(1)\right).
		\qedhere
	\end{align*}
\end{proof}

\begin{proof}[Proof of \cref{thm:main}]
	By \cref{prop:avg}, we have
	\( \E{\size{\BB{\KK_n}}} \ge t_n \).
	Set \( \varepsilon = 1 - \log 2 = 0.306\dotsm \, \).
	It follows from \cref{lem:equiv} that
	\begin{equation*}
		\log \E{\size{\BB{\KK_n}}} \ge \log t_n > \frac{\log^2 n}{\log 2}(1 - \varepsilon) = \log^2 n
	\end{equation*}
	for sufficiently large \( n \),
	which proves the theorem.
\end{proof}

\begin{table}[hbtp]
	\centering
	\begin{tabular}{lllllllllll}
		\toprule
		\( n \)        &
		\( 10^1 \)     &
		\( 10^2 \)     &
		\( 10^3 \)     &
		\( 10^4 \)     &
		\( 10^5 \)     &
		\( 10^6 \)     &
		\( 10^7 \)     &
		\( 10^8 \)     &
		\( 10^9 \)     &
		\( 10^{10} \)  \\
		\( \delta_n \) &
		\( 1.467 \)    &
		\( 0.860 \)    &
		\( 0.646 \)    &
		\( 0.566 \)    &
		\( 0.477 \)    &
		\( 0.416 \)    &
		\( 0.386 \)    &
		\( 0.347 \)    &
		\( 0.316 \)    &
		\( 0.299 \)    \\
		\bottomrule    \\
	\end{tabular}
	\caption{How large \( n \) should be for the theorem?}
	\label{tbl:diff}
\end{table}

In the end, we make a short comment on how large \( n \) should be for the theorem.
\cref{tbl:diff} shows the rounded values of
\begin{equation*}
	\delta_n = \size[\bigg]{\frac{\log t_n}{\log^2 n/\log 2} - 1}
\end{equation*}
for \( n = 10^1, \dotsc, 10^{10} \).
The proof indicates that \( n > 10^{10} \) would be sufficient for the theorem.

\section{Conclusions}
In this article, we addressed the problem of how large the average number of concepts is.
To this end, we introduced the distribution \( \dK_{n, p, q} \) for random contexts
and presented an exact formula for the average number \( \E{\size{\BB{\KK}}} \) of concepts of
a random context \( \KK \dist \dK_{n, p, q} \).
To establish a superpolynomial asymptotic lower bound,
random contexts with constant probabilities \( p = q = 1/2 \) were studied in detail.
For a sequence of random contexts \( (\KK_n) \) with \( \KK_n \dist \dK_{n, \frac{1}{2}, \frac{1}{2}} \),
we proved that
\( \E{\size{\BB{\KK_n}}} \gg n^{\log n} \).

\section*{Acknowledgments}
The author would like to thank Ken'ichi Kuga for his understanding of the preparation of this article.
The author would also like to thank Manabu Hagiwara for conducting several seminars on FCA and thank the participants:
Yuki Kondo, Hokuto Takahashi, and Hayato Yamamura.

\end{document}